\documentclass{article}

\usepackage[accepted]{icml2026}

\usepackage{microtype}
\usepackage{graphicx}
\usepackage{subcaption}
\usepackage{booktabs}
\usepackage{algorithm}
\usepackage{algorithmic}

\usepackage{amsmath,amsfonts,bm}

\def\eqref#1{equation~\ref{#1}}

\def\1{\bm{1}}

\DeclareMathAlphabet{\mathsfit}{\encodingdefault}{\sfdefault}{m}{sl}
\SetMathAlphabet{\mathsfit}{bold}{\encodingdefault}{\sfdefault}{bx}{n}

\usepackage{amsmath,amssymb,amsfonts,amsthm}
\usepackage{mathtools}
\usepackage{nicefrac,xcolor}
\usepackage{float}
\usepackage{pgfplots}
\pgfplotsset{compat=1.17}
\usepackage{placeins}
\usepackage{etoolbox}

\usepackage[capitalize,noabbrev]{cleveref}

\theoremstyle{plain}
\newtheorem{theorem}{Theorem}[section]
\newtheorem{lemma}[theorem]{Lemma}
\newtheorem{claim}[theorem]{Claim}
\newtheorem{corollary}[theorem]{Corollary}
\theoremstyle{definition}
\newtheorem{assumption}[theorem]{Assumption}
\theoremstyle{remark}

\icmltitlerunning{Efficient Attention via Pre-Scoring}

\begin{document}

\twocolumn[
\icmltitle{Efficient Attention via Pre-Scoring:\\
Prioritizing Informative Keys in Transformers}

\icmlsetsymbol{equal}{*}

\begin{icmlauthorlist}
  \icmlauthor{Zhexiang Li}{usc,equal}
  \icmlauthor{Haoyu Wang}{ucsd,equal}
  \icmlauthor{Yutong Bao}{ucd,equal}
  \icmlauthor{David P. Woodruff}{cmu}
\end{icmlauthorlist}

\icmlaffiliation{usc}{University of Southern California, Los Angeles, CA, USA}
\icmlaffiliation{ucsd}{University of California, San Diego, Department of Mathematics, La Jolla, CA, USA}
\icmlaffiliation{ucd}{University of California, Davis, Applied Mathematics \& Statistics, Davis, CA, USA}
\icmlaffiliation{cmu}{Carnegie Mellon University, Pittsburgh, PA, USA}

\icmlcorrespondingauthor{David P. Woodruff}{dwoodruf@cs.cmu.edu}

\icmlkeywords{Transformers, Efficient Attention, Long Context, Clustering, Leverage Scores}

\vskip 0.3in
]

\printAffiliationsAndNotice{
  \icmlEqualContribution\quad
  Code: \url{https://github.com/BruceLZX/prescored-transformer}
}

\begin{abstract}
Efficient attention mechanisms enable long-context transformers but often miss globally important tokens, degrading modeling quality. We introduce a pre-scoring framework that assigns a query-independent global importance prior to keys before applying hierarchical approximate attention. Using clustering-based or leverage-style scoring, pre-scoring identifies structurally informative keys and restricts computation to this prioritized subset. Integrated with HyperAttention, pre-scoring substantially improves approximation quality on long-context language modeling: on ChatGLM with 131k-token contexts, perplexity decreases from 12.0 to 9.5 under a fixed interaction budget while retaining subquadratic efficiency. Clustering-based scoring consistently outperforms leverage-based selection under identical key budgets. Beyond language, replacing self-attention in Vision Transformers preserves most of the baseline accuracy, showing that the approach generalizes across modalities. We provide structural guarantees under a planted-subspace model, showing that clustering recovers the same heavy-key sets as leverage-based methods. Overall, pre-scoring improves the efficiency–accuracy trade-off of approximate attention by better prioritizing informative keys without sacrificing scalability.
\end{abstract}

\section{Introduction}
Transformer-based large language models (LLMs) now underpin cutting-edge performance across diverse applications, including computer vision \cite{bi2021} and text classification \cite{Rodrawangpai2022}, yet their quadratic self-attention cost remains a persistent barrier to efficient long-context processing. Especially for Transformer-based attention mechanisms, the time complexity rises quadratically with respect to the sequence length.  Without effective strategies to cut this cost, deploying transformer models on tasks requiring massive contexts remains impractical and challenging. 

Consider the input of each attention layer, which is typically represented as an $n\times d$ matrix $X$, where $n$ denotes the context length and $d$ is the embedding dimension of the tokens. From this input matrix, we generate three distinct matrices by applying learned linear transformations. Specifically, we compute the query matrix $Q=X \cdot W_Q$, the key matrix $K=X \cdot W_K$, and the value matrix $V=X \cdot W_V $, where $W_Q$, $W_K$, and $W_V$ are learned parameter matrices of dimensions $d\times d$. Various methods have been developed to improve efficiency when processing $Q$, $K$, and $V$. For example, Performer \cite{choromanski2020} replaces the standard softmax attention with kernel methods. This remarkably reduces computational complexity, yet it involves potential kernel approximation-induced errors. FlashAttention \cite{dao2022} facilitates computation by optimizing memory access patterns during the matrix computations. Despite its advantages, specialized hardware support is required for optimal performance. 

More recent advances in efficient attention methods have explored diverse strategies to mitigate quadratic complexity while maintaining performance. For long-context processing, DuoAttention \cite{duoattention2024} combines retrieval-augmented heads with streaming attention mechanisms, dynamically balancing local and global context access. The adaptive approach in "Unveiling Simplicities of Attention" \cite{simplicities2024} identifies and preserves only the essential attention heads for long-context modeling, achieving comparable performance with significantly reduced computation. 

In addition, HyperAttention \cite{han2023} reduces the computational burden using approximate matrix product and locality sensitive hashing. These methods improve efficiency at the cost of some degradation in perplexity. LevAttention \cite{kannan2024} selects a fixed set of keys independent of the queries to try to capture all the heavy attention scores. However, it is less effective in capturing query-specific attention patterns.

Inspired by HyperAttention and LevAttention, we identify a missing structural component in existing efficient attention designs. Most approximating attention mechanisms focus on \emph{query-dependent locality}, which restrict each query to nearby or hashed keys through routing, LSH, or blockwise partitioning. While effective for reducing computation, such locality-based strategies may miss tokens that are globally influential across many queries.

We introduce a complementary axis: a query-independent global importance prior over keys. Before applying hierarchical attention, we perform a lightweight pre-scoring step that estimates which keys are structurally informative, using clustering-based or leverage-style ranking. Attention is then restricted to this prioritized subset, improving recall of high-contribution keys without reverting to quadratic computation. Pre-scoring differs from token pruning and dynamic sparsification methods that remove tokens or reduce sequence length. We do not shorten the sequence or discard representations; instead, we construct a lightweight structural prior over keys that biases which key-query interactions are evaluated by the downstream approximate attention kernel. In particular, pre-scoring is a front-end selector that improves the recall of globally influential keys under a fixed interaction budget, while query-dependent routing (e.g., LSH bucketing in HyperAttention) still determines which interactions are actually computed within the retained set. 
Our pre-scoring is a query-independent selector that provides a per-layer structural prior over keys, and is complementary to streaming KV-cache pruning approaches including heavy-hitter retention and attention-sink preservation that keep tokens based on accumulated attention statistics across decoding. Unlike KV-cache pruning, we do not remove tokens from the context; instead we restrict which key-query interactions are evaluated inside an approximate attention kernel (e.g., HyperAttention). A direct comparison to streaming-specific baselines is valuable future work, but is not strictly matched to our full-layer replacement protocol.

This perspective unifies clustering-based filtering and leverage score selection under a common principle: both approximate the subset of keys contributing the largest mass in the attention matrix. Unlike LevAttention, which uses algebraic sketches, our approach leverages geometric clustering that better matches the empirical distributional structure of key embeddings in trained transformers.

\paragraph{Our Results.}
We evaluate clustering-based pre-scoring both as a standalone key-selection mechanism and when integrated with HyperAttention.

First, we compare clustering against leverage-score selection under identical key budgets. On a ViT-Large model (ImageNet-1k), selecting 128 keys via K-means retains \textbf{84.46\%} accuracy, compared to \textbf{77.17\%} for leverage-score sampling. This empirical gap supports our theoretical picture that geometric clustering can provide a tighter global importance prior than algebraic sketch-based selection under the same key budget.

Second, we analyze pre-scoring within HyperAttention. Two orthogonal mechanisms influence performance: (i) HyperAttention’s blockwise implementation improves numerical stability and training dynamics;  (ii) Pre-scoring increases recall of high-contribution keys. To isolate these effects, Table~\ref{tab:diff} disentangles pre-scoring from blockwise optimization. Pre-scoring alone accounts for the majority of perplexity reduction, while the optimization flag provides complementary gains. The best absolute perplexity (9.5 on LongBench) uses both mechanisms, but pre-scoring independently improves the original HyperAttention pipeline from 17.54 to 10.38, demonstrating that the gains are not merely implementation artifacts.
\begin{table}[t]
\centering
\caption{Disentangling pre-scoring from blockwise optimization.}
\label{tab:diff}
\begin{tabular}{lccc}
\toprule
Method & Pre-score & Blockwise Opt. & PPL \\
\midrule
FlashAttention & False & False & 5.6 \\
HyperAttention & False & False & 17.54 \\
HyperAttention & False & True & 13.41 \\
K-means+Hyper & True & False & \textbf{10.38} \\
K-means+Hyper & True & True & \textbf{9.53} \\
\bottomrule
\end{tabular}
\end{table}

Our objective is not to minimize perplexity at any computational cost, but to improve approximation quality \emph{under constrained key budgets}. We therefore treat \emph{unfiltered HyperAttention} (i.e., pre-scoring disabled) as a \emph{high-compute reference point}: it incurs the largest number of key--query interactions and can achieve the lowest perplexity. In contrast, pre-scoring is designed to \emph{shift the accuracy--efficiency frontier} by achieving lower perplexity than HyperAttention at the \emph{same} retained-key budget.

To provide theoretical grounding, we revisit the planted-subspace model introduced for LevAttention and show that our clustering-based pre-scoring recovers the same heavy-key sets with high probability. Thus, our method matches leverage-score guarantees in structured settings while using geometric rather than algebraic importance estimation.

\section{Preliminaries}

\paragraph{Positioning of Our Approach.}
Efficient attention methods differ primarily in \emph{how they select key–query interactions}. Most existing approaches focus on \emph{query-dependent locality}, using mechanisms such as locality-sensitive hashing, routing, or blockwise partitioning to restrict attention to nearby tokens. In contrast, our method introduces a \emph{query-independent global importance prior} over keys. This prior estimates which keys are structurally informative across queries, and can be combined with locality-based mechanisms such as HyperAttention. Thus, our approach operates along a complementary axis to routing-based or hashing-based attention.

We define the attention mechanism $Att = D^{-1} A V$. Here, $A$ is defined as $A := \exp(QK^\top)$, while $D$ is a diagonal matrix with $D_{i,i} = \|A_{i,:}\|_1$ for each $i \in [n] = \{1, 2, \ldots, n\}$. We refer to the matrix $A$ as the attention matrix. Explicitly calculating the attention matrix requires $\Theta(n^2 d)$ time and  $\Theta(n^2)$ memory, which can be prohibitive.

Han et al. introduce HyperAttention~\cite{han2023}, which addresses the quadratic bottleneck of vanilla self-attention by hashing queries and keys with an angular locality-sensitive hash (LSH) function, and then ordering buckets so that adjacent buckets have a small Hamming distance. Scores are computed only for pairs that fall into the same hash bucket, significantly reducing the time and memory cost of typical data distributions. Their method performs an additional randomized low-rank compression inside each bucket, further reducing constant factors.  Its main limitation is data independence: because bucket membership is determined solely by the hash function, weak but semantically crucial long-range links may never collide and can therefore be missed.

LevAttention~\cite{kannan2024} takes a complementary view.  It first sketches the key matrix in $O(nd)$ time to approximate statistical leverage scores for each row of the key matrix, then forms a Universal set $U=\{i:\operatorname{LS}(K_i)\ge\epsilon\}$ that, for polynomial-based attention (rather than the standard softmax attention), is guaranteed to contain all attention scores whose weight exceeds a user-chosen threshold $\epsilon>0$. Looking only at $U$ therefore achieves perfect recall of heavy attention scores, independent of positional locality.  However, when $\epsilon$ is very small, $|U|$ can be large as $n$, eliminating any savings.

Moreover, a uniform evaluation within the set $U$ wastes time in many pairs with low impact. By first constructing $U$ to ensure theoretical coverage and then applying HyperAttention’s locality-sensitive hashing inside this recall-guaranteed set, our hybrid approach combines the strengths of both methods: it retains subquadratic complexity and provably captures every $\epsilon$-heavy attention.

Our goal is to accelerate the approximation methods for transformers while maintaining high accuracy. Rather than using the universal set $U$ of  \citet{kannan2024}, we suggest other sets based on clustering methods to guide approximation algorithms such as HyperAttention \cite{han2023} to find large attention scores.

\section{Algorithm}\label{para:algo}

A central motivation for our pre-scoring approach is the hypothesis that computationally efficient methods, such as clustering, can effectively identify and prioritize the most salient keys within the key matrix $K \in \mathbb{R}^{n \times d_k}$ (where $n$ is sequence length and $d_k$ is the key dimension). This idea is supported by Axiotis et al. \cite{axiotis2024}. They developed a cluster-based sensitivity sampling method to enhance the data selection efficiency for large-scale model training. This principle resonates with techniques such as LevAttention \cite{kannan2024}, which use statistical leverage scores to sample influential data points. To provide rigorous grounding for clustering, and specifically K-means or K-median) as a pre-scoring mechanism, we analyze its performance and relationship to leverage scores under a structured data model.

\subsection{Pre-Scored HyperAttention}


Algorithm \ref{alg:prescore} ranks the $n$ keys in a single
pass.  Given a stochastically perturbed matrix $K'$ of keys, it offers two
routes: (i) a one-shot $k$-means/$k$-median call that returns the $s$
closest keys to $k=d{+}1$ centroids, or (ii) a fast leverage-score sketch. We set the number of clusters to $d{+}1$: one centroid per latent orthogonal direction ($d$) plus a single residual bucket for noise/no-signal keys. In the planted–subspace model in algorithm section, this matches the ground-truth partition and guarantees the within–cluster variance term is $O(\sigma_S^2)$ while the between–cluster gap remains $\Omega(1)$ by Lemma \ref{lem2}. Theorems \ref{thm1} and \ref{thm2} guarantee that both routes isolate all $\Theta(\epsilon)$-heavy keys with exponentially small failure probability. The full model and assumptions are given after Algorithm \ref{alg:psha}. Lines 3–7 execute in $O(nd_k)$ time for clustering, and $O(nd_k\!\log d_k)$ for leverage scores. 

\begin{algorithm}[t]
\caption{PreScore: Rank Keys via Clustering/Leverage}
\label{alg:prescore}
\begin{algorithmic}[1] 
\REQUIRE Keys $K \in \mathbb{R}^{n \times d_k}$, clusters $k=d+1$, method $\in\{\textsc{Kmeans},\textsc{Kmedian},\textsc{Leverage}\}$
\STATE $K' \leftarrow K + \mathcal{N}(0,\sigma^2 I_{d_k})$ \COMMENT{optional noise}

\IF{$method = \textsc{Kmeans}$ \OR $method = \textsc{Kmedian}$}
  \STATE $\{C_j, \mu_j\}_{j=1}^k \leftarrow \textsc{Kmeans}(K',k)$
  \STATE $S \leftarrow$ indices of the $s$ keys nearest to their centroids
\ELSE
  \STATE $h \leftarrow \textsc{ApproxLeverage}(K')$
  \STATE $S \leftarrow$ top-$s$ indices by $h$
\ENDIF

\STATE \textbf{return} $S$
\end{algorithmic}
\end{algorithm}

\begin{algorithm}[t]
\caption{Pre-Scored HyperAttention}
\label{alg:psha}
\begin{algorithmic}[1]
\REQUIRE Query $Q$, Key $K$, Value $V$; clusters $k=d+1$; noise $\sigma$; fallback threshold $\delta$; method $\in \{\textsc{Kmeans}, \textsc{Kmedian}, \textsc{Leverage}\}$
\STATE $S \leftarrow \textsc{PreScore}(K,k,s,\sigma,method)$ \COMMENT{see Algorithm~\ref{alg:prescore}}
\IF{$|S| < \delta n$}
  \STATE \textbf{return} $\textsc{HyperAttention}(Q,K,V)$ \COMMENT{robust fallback}
\ENDIF
\STATE $Att_{\text{out}} \leftarrow \textsc{HyperAttention}(Q, K[S], V[S])$
\STATE \textbf{return} $Att_{\text{out}}$
\end{algorithmic}
\end{algorithm}

\paragraph{Computational and implementation perspective.}
Pre-scoring is a lightweight filtering step applied once per attention layer using only the current key matrix $K$. The clustering route costs $O(n d_k \cdot k \cdot I)$ for $I$ iterations (with $k\ll n$ in practice), while leverage-style ranking costs $O(n d_k \log d_k)$. We do not backpropagate through clustering and introduce no additional trainable parameters: pre-scoring is purely a deterministic key-selection module. For autoregressive decoding, pre-scoring is performed during the \emph{prefill} stage; during token-by-token decoding we reuse this selection (or update it only periodically), avoiding an $O(n)$ clustering pass at every step. In all clustering experiments, we run k-means for a fixed small number of iterations (at most $I=10$) per layer.

\section{Structural Guarantees}\label{para:theo}


We now provide theoretical justification for the pre-scoring mechanisms introduced in Section~\ref{para:algo}. Under a planted subspace model, we show that both clustering-based and leverage-score-based ranking recover keys with large attention contribution (heavy keys). These results explain why simple geometric clustering can approximate more expensive algebraic influence measures.
\subsubsection{Matrix Structure and Assumptions}\label{sec:matrix_structure}
\emph{The guarantees in this section follow \citet{kannan2024} and apply to polynomial attention; softmax results are empirical. Unless stated otherwise: 
\begin{assumption}[Model and regularity]\label{assump:model}
Keys split into $S\cup N$ with means $\mu_S,\mu_N$, within-cluster variance $\le\sigma^2$, separation $\Delta=\|\mu_S-\mu_N\|_2$; analysis uses a degree-$r$ polynomial kernel; the selector retains $s$ keys per query; all probability statements are over the data and algorithmic randomness.
\end{assumption}
}

We give a new planted model for which cluster-based prescoring recovers all large leverage scores. 

Let $A \in \mathbb{R}^{n \times d}$ be a matrix generated as follows:

\begin{enumerate}

    \item There are $d$ disjoint sets of row indices, $S_1, \dots, S_d$, each of size $m = \lceil 1/\epsilon \rceil$ for some small $\epsilon \in (0, 1)$.

    \item Let $S_0 = \{1, \dots, n\} \setminus \bigcup_{j=1}^d S_j$ be the set of remaining row indices. We assume $n \gg dm$, so $|S_0| = n(1-o(1))$. 

    \item Let $v_1, \dots, v_d \in \mathbb{R}^d$ be an orthonormal basis for $\mathbb{R}^d$.

    \item For each $j \in \{1, \dots, d\}$ and every $i \in S_j$, first draw an \emph{unnormalized} vector $\tilde{A}_i = v_j + \delta_{i,j}$ with $\delta_{i,j} \sim \mathcal{N}(0, \sigma_S^2 I_d)$ i.i.d.  We then normalize it and set $A_i = \tilde{A}_i / \|\tilde{A}_i\|_2$.

    \item For each $i \in S_0$, sample $\tilde{A}_i = \eta_i$ where $\eta_i \sim \mathcal{N}(0, \sigma_N^2 I_d)$ i.i.d., and again normalize via $A_i = \tilde{A}_i / \|\tilde{A}_i\|_2$.

    \item The noise scales satisfy $\sigma_S^2 = c_S/d$ and $\sigma_N^2 = c_N/(n\epsilon)$ for sufficiently small positive constants $c_S, c_N$. This implies that the noise variance within groups $d\sigma_S^2 = c_S$ and within the noise group $d\sigma_N^2 = d c_N/(n\epsilon)$ are small.
    \item \emph{Row‑norm regularity:} 
  \(\|A_i\|_2 = 1\) for all \(i\).
    \item The model explicitly states conditions on correlations:
    \begin{itemize}
        \item[($P_1$)] $\forall j,l \in S, j\ne l,|A_{j}A_{l}^{T}|\le\delta_{1}\cdot \min(||A_{j}||_{2}^{2},||A_{l}||_{2}^{2})$
        \item[($P_2$)] $\forall j\in S,l\notin S,|A_{l}A_{j}^{T}|\le\delta_{2}\cdot \min(||A_{j}||_{2}^{2},||A_{l}||_{2}^{2})$
    \end{itemize}
     We assume as in LevAttention \cite{kannan2024} that $\delta_1$ and $\delta_2$ are sufficiently small constants. 
     

\end{enumerate}
\paragraph{Remark}
The correlation bounds (P1)–(P2) do not control row norms. Appendix \ref{appb} gives a construction with
$\delta_1=\delta_2=0$ where a few noise rows have norm $M\!\gg\!1$. Their $M^2$-scaled contributions dominate the
$k$-means objective, “stealing” clusters from the signal set $S$ and preventing recovery despite perfect orthogonality.
Thus, without enforcing $\|A_i\|_2=1$ for all $i$, clustering can fail even in the best-case correlation regime.
For the analysis below we therefore adopt Assumption~\ref{assump:model} (row-norm regularity): all rows are
$\ell_2$-normalized, i.e., $\|A_i\|_2=1$. In practice, key vectors are produced from LayerNorm/RMSNorm-normalized hidden states followed by a linear map, which substantially reduces row-norm variation. In our implementation, we additionally $\ell_2$-normalize keys before applying clustering-based pre-scoring, aligning the empirical pipeline with the row-norm regularity used in the analysis and preventing the outlier-dominated failure mode in Appendix~\ref{appb}.

Recall that for any row $i$, its leverage score is
$
h_i = A_i\,(A^\top A)^{-1}A_i^\top = \sup_{\|x\|=1}\frac{(A_i^\top x)^2}{\|A x\|^2}.
$ 
Based on our assumption, we have the following:
\begin{lemma}[Upper Bound on Noise Leverage]\label{lem1}
In this model, for each row $i\in S_0$, consider that $||A_i||^2=1$ we have
$h_i \le \frac{\|A_i\|^2}{\sigma_{\min}^2} = \frac{1}{\Theta(1/\varepsilon)} = O(\varepsilon).$

\end{lemma}

\begin{lemma}[Lower Bound on Signal Leverage]\label{lem2}
For each row $i\in S_j$, choose unit $x = v_j$.  Then
$
h_i \ge \frac{(A_i^\top v_j)^2}{\|A v_j\|^2}
= \Theta(\epsilon).
$
\end{lemma}
The above lemmas are standard; we refer to the supplementary for their proofs. 
Letting $A$ be the key matrix $K$, we have:
 $h_j = \sup_{\|x\|=1} \frac{(k_j^\top x)^2}{\|Kx\|^2},$ and $\|Kx\|^2 \ge \sigma_{\min}^2 = \Theta(1)$.
\paragraph{Connection to real key matrices.}
The planted–subspace model should be viewed as an explanatory zoom-lens, not as a literal generative process for every transformer layer.  In a trained model, most key vectors are well spread across the unit sphere;
consequently, any two randomly chosen keys have almost orthogonal directions, and each heavy key tends to sit near a distinct “axis” of that sphere.  This geometric picture mirrors items (P1)–(P2) of our
assumptions, where signal rows are approximately orthogonal to both
noise rows and to one another.  Empirically, clustering with $k=d{+}1$ isolates one centroid per such axis plus a single centroid for the residual cloud of light keys, exactly as the model predicts. Thus the theoretical guarantees provide an intuitive explanation for why clustering-based pre-scoring works on real transformers even when the data are only approximately, rather than exactly, in planted-subspace position.

\begin{theorem}[Leverage‐Score Separation]\label{thm1}
Let $A\in\mathbb{R}^{n\times d}$ satisfy Assumption~1 ($\|A_i\|_2=1$ for all $i$) and (P1)--(P2).
Assume the concentration $\sigma_{\min}^2(A^\top A)=\Theta(1/\varepsilon)$.
Let $S$ be the set of signal rows and $N=[n]\setminus S$ the noise rows.
Then there exist constants $0<C_{\mathrm{noise}}<C_{\mathrm{sig}}$,
depending only on the model parameters, such that with probability at least $1 - 1/\mathrm{poly}(d)$

\[
\max_{i\in N} h_i \le C_{\mathrm{noise}}\varepsilon
\qquad\text{and}\qquad
\min_{i\in S} h_i \ge C_{\mathrm{sig}}\varepsilon.
\]
Consequently, any threshold $\tau\in(C_{\mathrm{noise}}\varepsilon,\ C_{\mathrm{sig}}\varepsilon)$
perfectly separates noise from signal by leverage scores.
\end{theorem}


\begin{theorem}[K-means Clustering]\label{thm2}

Under the same assumptions as Theorem \ref{thm1}, and assuming $c_S$ and $c_N$ are sufficiently small (e.g., $c_S < 1/2$ and $d c_N/(n\epsilon) < 1/2$), with probability at least $1 - \exp(-\Omega(\min(m, n-dm, d)))$), the $k$-means algorithm with $k=d+1$,  applied to the rows of $A$ converges to a clustering where:

\begin{enumerate}

    \item There are $d$ clusters, $C_1, \dots, C_d$, such that for each $j \in \{1, \dots, d\}$, all rows in $S_j$ are assigned to cluster $C_j$. The centroid $\mu_j$ of $C_j$ satisfies $\|\mu_j - v_j\| = O(\sigma_S/\sqrt{m})$.

    \item There is one cluster $C_0$ containing all rows from $S_0$. The centroid $\mu_0$ of $C_0$ satisfies $\|\mu_0 - \mathbf{0}\| = O(\sigma_N/\sqrt{n-dm})$.

\end{enumerate}

\end{theorem}

Running K-means with $k = d+1$ clusters on $\{k_j\}$ yields centroids $\mu_1, \dots, \mu_d, \mu_{d+1}$. Under the above separations, an optimal solution aligns with $C_i = S_i$ for $i = 1, \dots, d$ and $C_{d+1} = S$, since
$\sum_{k_j \in S_i} \|k_j - u_i\|^2 = 0, \quad \sum_{s \in S} \|s - \mu_{d+1}\|^2 \le \sum_{s \in S} \|s - \mu\|^2 \quad \forall \mu$,
and any deviation incurs a large additional within-cluster sum of squares (WSS).
In fact, by Theorem~\ref{thm1} the true partition $C_j=S_j$ ($j=1,\dots,d$) and $C_{d+1}=S_0$ has total within‐cluster cost
\[
\begin{aligned}
&\sum_{j=1}^d \sum_{i\in S_j}\|A_i - v_j\|^2 + \sum_{i\in S_0}\|A_i\|^2 \\
&= O(m\,\sigma_S^2) + O((n-dm)\,\sigma_N^2) \\
&= o(1)\,,
\end{aligned}
\]
whereas moving any single point to the wrong cluster incurs an additional penalty of at least
$
\min\{\|v_j-v_k\|^2,\|v_j\|^2\} - o(1) = 1 - o(1)> 0.
$
Hence any mis-assignment increases the total sum‐of‐squares, making the true grouping the unique global minimizer of the $k$‐means objective. Thus, clustering identifies those rows with $h_j \ge \epsilon$, matching LevAttention's heavy-key selection.

\begin{corollary}[Singleton case of Theorem 2]\label{corollary_1}
Setting \(m=1\) (so $\epsilon=1$ in this special case) in Theorem \ref{thm2} shows that, with the same
probability, the optimal \(k\)-means clustering with \(k=d+1\) places every
signal row in its own cluster and gathers all noise rows in \(C_0\).
\end{corollary}


\begin{proof}[Proof Sketch]
With \(m=1\), each heavy row \(K_j\in S\) contributes zero within‑cluster distortion when isolated.  If instead \(K_j\) is merged with any other row \(K_x\), the centroid error lower bound
\[
  \|K_j - \mu\|_2^2 \;\ge\; \tfrac12\|K_j - K_x\|_2^2 
  \;\ge\; \tfrac12D_{\min}
\]
applies, where \(D_{\min}>0\) is the minimum inter‑point separation guaranteed by Theorem~\ref{thm1}.  Since \(D_{\min}/2\) remains a positive constant, any such mis‑assignment strictly increases the total \(k\)‑Means cost.  Finally, choosing \(k=d+1\) reserves one centroid per heavy key and one for all noise rows, matching the leverage‑score separation. For a detailed proof, see Appendix \ref{appd} in the supplementary material.
\end{proof}

\paragraph{Why the assumptions are approximately satisfied in Transformers.}
The planted--subspace model is an explanatory abstraction rather than a literal generative process for every layer. However, several architectural features make its key assumptions reasonable in practice.First, in many Transformer blocks, keys are produced from LayerNorm-normalized hidden states (or other normalization variants) followed by a linear map; this discourages large row-norm variation and mitigates the failure mode in Appendix~\ref{appb}, which shows that $k$-means can be dominated by a small number of high-norm outliers even when correlations are tiny.Second, empirically, ``heavy'' attention often concentrates on a relatively small set of semantically salient tokens (e.g., delimiters, anchors, entities, or long-range cues), which can be idealized as a small number of prominent directions,
while the remaining keys form a diffuse background cloud; this matches the signal-noise decomposition $S\cup N$. Third, high-dimensional embeddings are typically close to orthogonal in aggregate, so approximate versions of (P1)--(P2)
are a plausible coarse model. Consequently, the structural guarantees below should be interpreted as explaining why a geometric proxy such as clustering can recover heavy-key sets under mild regularity, not as requiring exact planted structure.

\paragraph{Connection to the planted model and the constant gap $D_{\min}$.}
The Gaussian planted model studied in Section 4 of LevAttention\cite{kannan2024} draws signal and noise rows from two covariance profiles whose variances differ by a fixed ratio.  Standard concentration shows that this forces a constant lower bound 
$$
D_{\min}=2(1-\vartheta_1)\!\ge\!\tfrac32
$$
on the squared Euclidean distance between any signal and any noise row whenever the variance-ratio parameter satisfies $\vartheta_1<\tfrac14$.  That same constant gap is exactly what drives Corollary~\ref{corollary_1}: putting a signal row into a mixed cluster would increase the $k$-means objective by at least $D_{\min}/2$, so the optimal solution (with $k=d+1$) must isolate every signal row and pool all noise rows in $C_0$. Thereby, the planted model supplies a probabilistic guarantee for the deterministic separation our corollary needs, confirming that the singleton clustering phenomenon emerges naturally whenever the signal-to-noise variance ratio is sufficiently small.

In addition to our theorems above, this framework extends naturally to any $\ell_p$ norm where $p > 0$, allowing us to generalize our pre-scoring mechanism. Consider a key matrix $K = \{k_j\}_{j=1}^n \subset \mathbb{R}^{d_k}$ drawn from a mixture of $d$ well-separated “heavy” centers and a bulk of “light” points, as in our structured data model. We define the $\ell_p$-sensitivity of each key $k_j$ following \cite{padmanabhan2023}, which quantifies the importance of keys under the $\ell_p$ norm.

To adapt our clustering approach, we employ Minkowski-$k$-means, which minimizes the following:
$\sum_{j=1}^n \min_{i \in [d]} \|k_j - \mu_i\|_p^p.$ This method clusters keys by minimizing the $p$-th power of their $\ell_p$ distances to the nearest centroid. Under the same separation conditions as in Lemma \ref{lem2}, Minkowski-$k$-means\cite{oti2021} accurately recovers the top-heavy keys according to their $\ell_p$-sensitivity. Specifically, we show:
\\
\begin{claim}[$\ell_p$-Generalization]\label{claim1}
Under the structured planted model mentioned above (with disjoint signal sets $S_1,\dots,S_d$, noise set $S_0$, and  
$A_i = v_j + \delta_{i,j}$ for $i\in S_j$, $A_i=\eta_i$ for $i\in S_0$, and noise scales  
$\sigma_S^2=c_S/d$, $\sigma_N^2=c_N/(n\epsilon)$), running $k$-means with the $\ell_p$ metric (i.e.,\ using distances $\|x - y\|_p^p$) and $k=d+1$ recovers exactly the true clusters $S_1,\dots,S_d,S_0$, provided $c_S$ and $c_N$ are sufficiently small.
\end{claim}

For a brief proof of this claim, we replace every squared-norm ($p=2$) in the analysis of lemma~\ref{lem1} with the $p$-th power norm, and the separation conditions ensure that the heavy centers remain distinguishable. Based on our proof in \ref{claim1_proof}, we conclude:
In the $\ell_p^p$ metric one checks that
$
\|v_j - 0\|_p^p = 1$
and 
$\|v_j - v_k\|_p^p =2$
 for $j\neq k$ so the minimum inter-centroid separation is $\Delta_{\min}=1$, while standard bounds show that each point’s $p$-th power deviation $\delta_{\max}=O(c_S^{p/2}d^{1-p/2})=o(1)$. Hence any mis-assignment raises the $k$-means cost by at least $\Delta_{\min}-2\delta_{\max}>0$, making the true partition the unique global minimizer.
Since $\delta_{\max}=o(1)$ holds with probability at least $1-\exp(-\Omega(\min(m,n-dm)))$, $\ell_p$-$k$-means recovers the true clusters with probability at least $1 - 1/\mathrm{poly}(d)$.

This generalization enables our method to prioritize informative keys under various $\ell_p$ metrics, which may be advantageous for different data distributions or model architectures. By restricting the queries to attend only to the sampled keys selected by the method above, we ensure focused attention mechanisms while preserving the framework's adaptability. We provide both deterministic and probabilistic analyses of this connection in Appendix \ref{appc} of supplementary material.
\section{Experimental Evaluations}
To evaluate our proposed pre-scoring attention algorithm, we conducted comprehensive experiments assessing runtime efficiency, perplexity performance, and applicability to Vision Transformers (ViTs) \cite{dosovitskiy2021vit}. We compared three variants—K-means+Hyper, K-median+Hyper, and Lev+Hyper—against baseline HyperAttention \cite{han2023} and FlashAttention \cite{dao2022}. We also analyzed the key clustering performance of K-means relative to ViT's standard self-attention. All the experiments were done on a single NVIDIA A100 GPU with 40 GB memory or a single NVIDIA L4 GPU with 24 GB memory.

\subsection{Speed Comparison on Flash Attention}

    

\begin{figure}[t]
    \centering
    \begin{subfigure}{\columnwidth}
        \centering
        \includegraphics[width=0.9\linewidth]{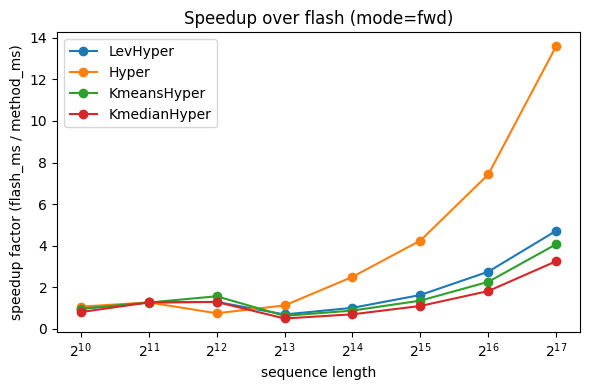}
        \caption{Forward pass only.}
        \label{fig:image1}
    \end{subfigure}

    \vspace{1mm}

    \begin{subfigure}{\columnwidth}
        \centering
        \includegraphics[width=0.9\linewidth]{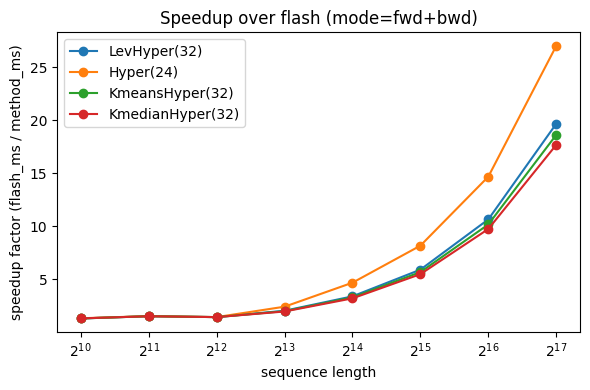}
        \caption{Forward + backward pass.}
        \label{fig:speed_fwdbwd}
    \end{subfigure}

    \caption{Single-layer speedup over FlashAttention for HyperAttention and pre-scored variants.}
    \label{fig:image2}
\end{figure}

 FlashAttention is the gold-standard for exact softmax attention throughput, so we report speed relative to it. HyperAttention already provides substantial wall-clock acceleration at long context lengths, and we examine whether pre-scoring methods preserve this advantage while improving modeling quality.
 
 We tested the speed-up factor for each layer compared to that of Flash Attention. This speedup factor is the ratio of Flash Attention's runtime to that of each tested method on a per-layer basis. From the results, all combinations, Lev + Hyper, K-means + Hyper and K-median + Hyper, outperform Flash Attention for sufficiently long sequences, demonstrating the advantage of HyperAttention-based methods. Compared to the original HyperAttention, these methods can generate mild acceleration, with the performance becoming more remarkable from $N=2^{13}$ with a speedup factor of around $3$ to $4$ in Figure~\ref{fig:image1}. Such tradeoffs are from the time complexity of the pre-scoring algorithm in forward selection parts, which is $O(N\cdot d^2)$ for LevAttention method\cite{kannan2024} and $O(N\cdot d \cdot k)$ for K-means/median, where $k$ is the number of clusters in the K-means/median part. Because of the huge size of the key matrix, having $d \gg k$ is a common situation. In this case, it is also reasonable to conclude that the additional complexity is roughly $O(N\cdot d)$, a near-linear complexity dependent solely on the dimensions of $K$. Comparing these two options, Lev+Hyper exhibits the best scalability, closely matching HyperAttention on longer sequences, while K-median+Hyper requires slightly higher computational cost due to its clustering complexity. Our pre-scoring overhead is most pronounced in the forward pass, as the backward pass adheres to HyperAttention’s standard pipeline. As a result, extending pre-scoring to both passes could potentially narrow the speedup factor.

\subsection{Accuracy Comparison}

\paragraph{Corrected Coupling (GLM3).}
After resolving residual coupling artifacts present in early experiments, we evaluate the corrected pre-scoring integration (GLM3). Figure~\ref{fig:image4} shows that perplexity decreases monotonically as the number of retained keys increases, consistent with the efficiency–accuracy tradeoff discussed in this section.

While their speedup factors grow more gradually, our HyperAttention-based methods offer a balanced trade-off between runtime and perplexity, as observed in the speed comparison. Similarly to HyperAttention's test \cite{han2023}, we use the \textit{LongBench} dataset \cite{bai2023longbench} and evaluate on \textit{ChatGLM2-6B-32k}~\cite{zeng2024chatglm2} and \textit{ChatGLM3-6B-32k}~\cite{zeng2024chatglm3}. Hereafter we refer to them as \textbf{GLM2} and \textbf{GLM3}, respectively. To compare the efficacy of different methods under full-layer replacement, we incorporate K-means/K-median and LevAttention with our scoring mechanism before sending the scored data to the HyperAttention algorithm.

To illustrate the sensitivity of approximate attention to masking/coupling choices, we include an ablation on ChatGLM2 in Appendix~\ref{appf}; all main-text results use the corrected integration (GLM3).

\paragraph{Results on ChatGLM3-6B-32k.}
We evaluate our method on \textit{ChatGLM3-6B-32k} under the same LongBench/full-layer protocol as GLM2. To avoid implementation artifacts observed on GLM2, we use a corrected coupling of pre-scoring and HyperAttention (details in Appendix~\ref{appf}). All main-text results use a corrected integration of pre-scoring with HyperAttention. 
Concretely, we (i) apply the pre-scoring selection via an attention bias mask rather than zeroing keys/values, 
(ii) scale residual Monte-Carlo samples by the effective retained key count, and (iii) exclude blockwise-selected keys from the residual path to avoid double counting. This preserves the geometry of the key space and yields stable perplexity–budget curves (Figure~\ref{fig:image4}).

\begin{figure}[h!]
    \centering
    \graphicspath{{./pictures/}}  
    \includegraphics[width=\linewidth]{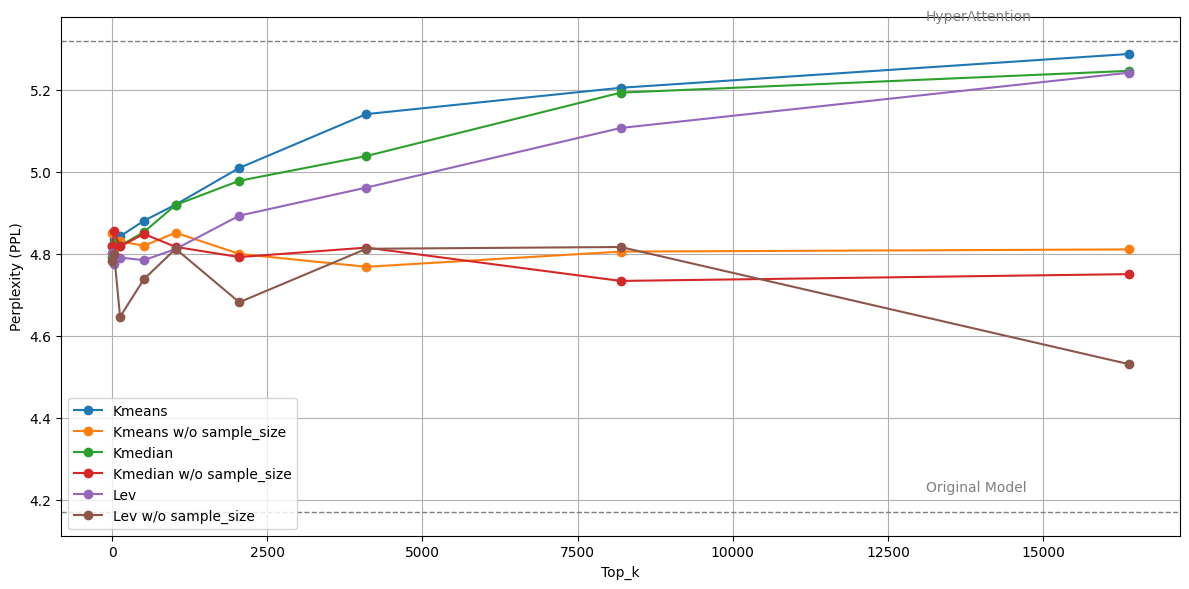}
    \caption{GLM3: perplexity vs.\ top-$k$ for K-means/K-median/Leverage, with/without residual sampling.}
    \label{fig:image4}
\end{figure}

Some vital observations are listed here:
\begin{samepage}
\begin{itemize}
    \item \textbf{Small $k$ saturates:} Curves are flat at low $k$ and lie \emph{below} vanilla HyperAttention, indicating pre-scoring mainly \emph{denoises} and a few hundred keys capture most mass.
    \item \textbf{Residual vs.\ no-residual:} With residual sampling, curves increase mildly as $k$ grows; without residuals they stay stable with small dips at very low $k$, consistent with block-diagonal attention already capturing most mass.
    \item \textbf{Runtime:} Lev scoring is light and masking is applied as bias, so runtime varies weakly with $k$; K-means/K-median incur more overhead as $k$ increases.
\end{itemize}
\end{samepage}
\paragraph{Efficiency–Accuracy Tradeoff.}
The lowest perplexity is observed when $k=0$, where no filtering is applied and the full HyperAttention pipeline is used. However, this regime also incurs the highest computational cost. The purpose of pre-scoring is to improve modeling quality \emph{at reduced key budgets}, not to outperform the full-compute baseline. 

Across a wide range of $k$, our method achieves substantially lower perplexity than HyperAttention operating under the same key constraints. In this sense, pre-scoring shifts the accuracy–efficiency frontier, enabling better performance for a given computational budget.

\subsection{Zero-Shot Substitution Vision-Transformer}
We replaced the standard softmax-based self-attention layers in the Vision Transformer (ViT) \cite{dosovitskiy2021vit} with our proposed K-means sampling attention mechanism by letting queries $Q$ only attend to a subset $S$ of keys $K$ chosen by our algorithm \ref{alg:prescore}. We also mask the value matrix $V$ with our subset $S$ to align with the original output shape. For the baseline, we used the pretrained \textit{vit\_small\_patch16\_224} and \textit{vit\_large\_patch16\_224} models, which achieved top-1 accuracies of 85.11\% and 85.85\%, respectively, in the ImageNet-1k validation set \cite{deng2009imagenet}. Our custom attention mechanism replaces the original full computation with a K-means clustering approach that samples a subset of key vectors per head. We varied the number of clusters and sampled keys to evaluate the performance. In the ViT-Small variant, when the number of clusters was fixed at 4, reducing the number of sampled keys to 32 resulted in a drastic accuracy drop (31.34\%), while increasing the sample count to 64, 96, and 128 gradually improved the accuracy to 61.31\%, 74.21\%, and 80.05\%, respectively. A further increase to 6 clusters with 128 sampled keys yielded a marginal improvement (80.49\%). In contrast, the ViT-Large variant exhibited higher robustness: with 4 clusters, the accuracy improved from 53.05\% with 32 samples to 78.10\%, 82.89\%, and 84.46\% for 64, 96, and 128 samples, respectively, while increasing the cluster count to 6 with 128 sampled keys maintained an accuracy of 84.46\%. These results, summarized in Table~\ref{tab:vit_results}, indicate that our K-means sampling attention can closely approximate the performance of the full self-attention mechanism if a sufficient number of key vectors is sampled. Our method achieves better performance than LevAttention \cite{kannan2024} with pretrained ViT models and share similar accuracy trade-offs to LevAttention with models trained from scratch using their updated leverage-score based attention mechanism. More detailed results of LevAttention are included in Appendix \ref{appe}.

\begin{table}[h!]
\centering
\caption{Accuracy of zero-shot substitution ViT with K-means prescoring (higher is better).}
\label{tab:vit_results}
\small
\setlength{\tabcolsep}{6pt}
\begin{tabular}{@{}lcc@{}}
\toprule
\textbf{Configuration} & \textbf{S/16 Acc.} & \textbf{L/16 Acc.} \\
\midrule
Base model                   & 85.11\% & 85.85\% \\
num\_cluster=4,\; num\_sample=32   & 31.34\% & 53.05\% \\
num\_cluster=4,\; num\_sample=64   & 61.31\% & 78.10\% \\
num\_cluster=4,\; num\_sample=96   & 74.21\% & 82.89\% \\
num\_cluster=4,\; num\_sample=128  & 80.05\% & \textbf{84.46\%} \\
num\_cluster=6,\; num\_sample=128  & 80.49\% & \textbf{84.46\%} \\
\bottomrule
\end{tabular}
\end{table}

\paragraph{Scope of baselines.}
Our experiments focus on full-layer replacement within the HyperAttention pipeline and direct comparison to LevAttention, which are the most structurally aligned baselines for query-independent key prioritization under a fixed interaction budget. Methods such as DuoAttention and streaming KV-cache pruning target different regimes (retrieval-augmented heads, streaming constraints, or cache eviction policies) and are not directly matched to our evaluation protocol. We therefore treat comparisons to streaming/pruning baselines as complementary future work and include a discussion in Appendix~\ref{apph}.

\section{Conclusion}
We introduced a pre-scoring mechanism that adds a query-independent importance prior to hierarchical attention. By selectively prioritizing informative keys, we overcome limitations of HyperAttention's uniform residual sampling, achieving significant perplexity improvements and computational efficiency advantages over FlashAttention\cite{dao2022}. Empirical results across both models validate our method’s effectiveness and generality. Mathematically, we also used planted-subspace model to prove that our clustering-based scoring is as powerful as leverage score pre-scoring for natural models.

\bibliography{icml2026}

\begin{thebibliography}{16}
\providecommand{\natexlab}[1]{#1}
\providecommand{\url}[1]{\texttt{#1}}
\expandafter\ifx\csname urlstyle\endcsname\relax
  \providecommand{\doi}[1]{doi: #1}\else
  \providecommand{\doi}{doi: \begingroup \urlstyle{rm}\Url}\fi

\bibitem[Axiotis et~al.(2024)Axiotis, Cohen-Addad, Henzinger, Jerome, Mirrokni, Saulpic, Woodruff, and Wunder]{axiotis2024}
Axiotis, K., Cohen-Addad, V., Henzinger, M., Jerome, S., Mirrokni, V., Saulpic, D., Woodruff, D., and Wunder, M.
\newblock Data-efficient learning via clustering-based sensitivity sampling: Foundation models and beyond, 2024.
\newblock URL \url{https://arxiv.org/abs/2402.17327}.

\bibitem[Bai et~al.(2024)Bai, Lv, Zhang, Lyu, Tang, Huang, Du, Liu, Zeng, Hou, Dong, Tang, and Li]{bai2023longbench}
Bai, Y., Lv, X., Zhang, J., Lyu, H., Tang, J., Huang, Z., Du, Z., Liu, X., Zeng, A., Hou, L., Dong, Y., Tang, J., and Li, J.
\newblock Longbench: A bilingual, multitask benchmark for long context understanding, 2024.
\newblock URL \url{https://arxiv.org/abs/2308.14508}.

\bibitem[Bi et~al.(2021)Bi, Zhu, and Meng]{bi2021}
Bi, J., Zhu, Z., and Meng, Q.
\newblock Transformer in computer vision.
\newblock In \emph{2021 IEEE International Conference on Computer Science, Electronic Information Engineering and Intelligent Control Technology (CEI)}. IEEE, 2021.

\bibitem[Choromanski et~al.(2022)Choromanski, Likhosherstov, Dohan, Song, Gane, Sarlos, Hawkins, Davis, Mohiuddin, Kaiser, Belanger, Colwell, and Weller]{choromanski2020}
Choromanski, K., Likhosherstov, V., Dohan, D., Song, X., Gane, A., Sarlos, T., Hawkins, P., Davis, J., Mohiuddin, A., Kaiser, L., Belanger, D., Colwell, L., and Weller, A.
\newblock Rethinking attention with performers, 2022.
\newblock URL \url{https://arxiv.org/abs/2009.14794}.

\bibitem[Dao et~al.(2022)Dao, Fu, Ermon, Rudra, and R{\'e}]{dao2022}
Dao, T., Fu, D.~Y., Ermon, S., Rudra, A., and R{\'e}, C.
\newblock Flashattention: Fast and memory-efficient exact attention with io-awareness.
\newblock In \emph{Advances in Neural Information Processing Systems (NeurIPS)}, 2022.

\bibitem[Deng et~al.(2009)Deng, Dong, Socher, Li, Li, and Fei-Fei]{deng2009imagenet}
Deng, J., Dong, W., Socher, R., Li, L.-J., Li, K., and Fei-Fei, L.
\newblock Imagenet: A large-scale hierarchical image database.
\newblock In \emph{Proceedings of the IEEE Conference on Computer Vision and Pattern Recognition (CVPR)}, pp.\  248--255. IEEE, 2009.

\bibitem[Donhauser et~al.(2025)Donhauser, Arnal, Pezeshki, Cabannes, Lopez-Paz, and Ahuja]{simplicities2024}
Donhauser, K., Arnal, C., Pezeshki, M., Cabannes, V., Lopez-Paz, D., and Ahuja, K.
\newblock Unveiling simplicities of attention: Adaptive long-context head identification, 2025.
\newblock URL \url{https://arxiv.org/abs/2502.09647}.

\bibitem[Dosovitskiy et~al.(2021)Dosovitskiy, Beyer, Kolesnikov, Weissenborn, Zhai, Unterthiner, Dehghani, Minderer, Heigold, Gelly, Uszkoreit, and Houlsby]{dosovitskiy2021vit}
Dosovitskiy, A., Beyer, L., Kolesnikov, A., Weissenborn, D., Zhai, X., Unterthiner, T., Dehghani, M., Minderer, M., Heigold, G., Gelly, S., Uszkoreit, J., and Houlsby, N.
\newblock An image is worth 16x16 words: Transformers for image recognition at scale.
\newblock In \emph{International Conference on Learning Representations (ICLR)}, 2021.

\bibitem[GLM et~al.(2024)GLM, :, Zeng, Xu, Wang, Zhang, Yin, Zhang, Rojas, Feng, Zhao, Lai, Yu, Wang, Sun, Zhang, Cheng, Gui, Tang, Zhang, Sun, Li, Zhao, Wu, Zhong, Liu, Huang, Zhang, Zheng, Lu, Duan, Zhang, Cao, Yang, Tam, Zhao, Liu, Xia, Zhang, Gu, Lv, Liu, Liu, Yang, Song, Zhang, An, Xu, Niu, Yang, Li, Bai, Dong, Qi, Wang, Yang, Du, Hou, and Wang]{zeng2024chatglm2}
GLM, T., :, Zeng, A., Xu, B., Wang, B., Zhang, C., Yin, D., Zhang, D., Rojas, D., Feng, G., Zhao, H., Lai, H., Yu, H., Wang, H., Sun, J., Zhang, J., Cheng, J., Gui, J., Tang, J., Zhang, J., Sun, J., Li, J., Zhao, L., Wu, L., Zhong, L., Liu, M., Huang, M., Zhang, P., Zheng, Q., Lu, R., Duan, S., Zhang, S., Cao, S., Yang, S., Tam, W.~L., Zhao, W., Liu, X., Xia, X., Zhang, X., Gu, X., Lv, X., Liu, X., Liu, X., Yang, X., Song, X., Zhang, X., An, Y., Xu, Y., Niu, Y., Yang, Y., Li, Y., Bai, Y., Dong, Y., Qi, Z., Wang, Z., Yang, Z., Du, Z., Hou, Z., and Wang, Z.
\newblock Chatglm: A family of large language models from glm-130b to glm-4 all tools, 2024.
\newblock URL \url{https://arxiv.org/abs/2406.12793}.

\bibitem[Han et~al.(2023)Han, Jayaram, Karbasi, Mirrokni, Woodruff, and Zandieh]{han2023}
Han, I., Jayaram, R., Karbasi, A., Mirrokni, V., Woodruff, D.~P., and Zandieh, A.
\newblock Hyperattention: Long-context attention in near-linear time, 2023.
\newblock URL \url{https://arxiv.org/abs/2310.05869}.

\bibitem[Kannan et~al.(2024)Kannan, Bhattacharyya, Kacham, and Woodruff]{kannan2024}
Kannan, R., Bhattacharyya, C., Kacham, P., and Woodruff, D.~P.
\newblock Levattention: Time, space, and streaming efficient algorithm for heavy attentions, 2024.
\newblock URL \url{https://arxiv.org/abs/2410.05462}.

\bibitem[Oti et~al.(2021)Oti, Olusola, Oberhiri-Orumah, and Nwankwo]{oti2021}
Oti, E.~U., Olusola, M.~O., Oberhiri-Orumah, G., and Nwankwo, C.~H.
\newblock New k-means clustering method using minkowski's distance as its metric.
\newblock \emph{British Journal of Computer, Networking and Information Technology}, 4\penalty0 (1):\penalty0 28--41, 2021.
\newblock \doi{10.52589/BJCNIT-XEPSJBWX}.

\bibitem[Padmanabhan et~al.(2023)Padmanabhan, Woodruff, and Zhang]{padmanabhan2023}
Padmanabhan, S., Woodruff, D.~P., and Zhang, Q.
\newblock Computing approximate $\ell_p$ sensitivities, 2023.
\newblock URL \url{https://arxiv.org/abs/2311.04158}.

\bibitem[Rodrawangpai \& Daungjaiboon(2022)Rodrawangpai and Daungjaiboon]{Rodrawangpai2022}
Rodrawangpai, B. and Daungjaiboon, W.
\newblock Improving text classification with transformers and layer normalization.
\newblock \emph{Machine Learning with Applications}, 10:\penalty0 100403, 2022.

\bibitem[Xiao et~al.(2024)Xiao, Tang, Zuo, Guo, Yang, Tang, Fu, and Han]{duoattention2024}
Xiao, G., Tang, J., Zuo, J., Guo, J., Yang, S., Tang, H., Fu, Y., and Han, S.
\newblock Duoattention: Efficient long-context llm inference with retrieval and streaming heads, 2024.
\newblock URL \url{https://arxiv.org/abs/2410.10819}.

\bibitem[{ZAI-ORG}(2024)]{zeng2024chatglm3}
{ZAI-ORG}.
\newblock Chatglm3-6b-32k.
\newblock Hugging Face model card, 2024.
\newblock URL \url{https://huggingface.co/zai-org/chatglm3-6b-32k}.
\newblock Accessed: 2025-09-24.

\end{thebibliography}
\bibliographystyle{icml2026}

\newpage
\appendix

\section{Appendix A: Perplexity Comparison Across Configuration}\label{appa}

In this section, we give tables to represent the data in our experiments for comparing accuracy. Both K-means and K-median prescoring achieve their best performance at \verb|top-k = 2048| with sampling. Lev+Hyper method reaches best its PPL performance at \verb|top-k=8192|. If we allow \verb|min-seq-len>=n_query|, their best perplexity will improve to about 9.5 at \verb|top-k=8192|. We also confirmed that the \verb|top-k| corresponding to the best PPL performance is between 2048 and 8192. 
Also, when we set \verb|sample_size=0|, the experiments show the best perplexity when \verb|top_k=0|. Given these two conditions, our accuracy further improves to $8.3081$ for all filtering methods when \verb|min-seq-len>=n_query|, which further improves to 30.8\% compared to HyperAttention.

When pre-scoring is disabled (unfiltered HyperAttention: \verb|top_k|=0, \verb|sample_size|=0, we observe the lowest perplexity. Under the additional setting \verb|min_seq_len >= n_query|, perplexity reaches 8.3081; this gain is attributable to bypassing the fallback path / enabling blockwise optimization, not to pre-scoring. For this reason, we do not claim 8.3081 as a pre-scoring improvement, and we evaluate pre-scoring primarily under constrained key budgets.”

\subsection{On the U-Shaped Performance Curve and \texttt{top\_k=0} Result}
We observe a U-shaped performance curve in our PPL experiments (see Tables~\ref{tab:ppl_K-means}, \ref{tab:ppl_K-median}, \ref{tab:ppl_lev}). This can be explained by a trade-off between capturing sufficient signal and introducing noise.
\begin{itemize}
    \item At a \textbf{low \texttt{top-k}}, the model fails to select a sufficient number of informative keys, leading to higher perplexity as crucial information is missed.
    \item At a \textbf{very high \texttt{top-k}} (e.g., 16384), the pre-scoring selects an excessive number of keys. This can introduce noise and less relevant information that degrades the performance of the subsequent HyperAttention stage, causing the perplexity to rise again. 
    \item The \textbf{optimal performance} is achieved at a balance point (empirically found between 2048 and 8192) where the most salient keys are captured without introducing excessive noise.
\end{itemize}
The strong performance at \texttt{top\_k=0} with \texttt{sample\_size=0} and \verb|min-seq-len>=n_query| is a special case. Here, pre-scoring is deactivated, and the model relies solely on the original HyperAttention mechanism. The performance gain to 8.3081 comes from the \verb|min-seq-len>=n_query| setting, which ensures the model fully utilizes blockwise optimizations even at shorter sequence lengths, rather than from the pre-scoring itself.

\begin{table}[h!]
\centering
\caption{PPL comparison for \textbf{K-means} across configurations. }
\small
\begin{tabular}{lrrrr}
\toprule
\textbf{Top K} & \textbf{Sample Size} & \textbf{PPL} & \textbf{PPL\textsuperscript{*}} \\
\midrule
0 & 256 & 17.5419 & 13.4143 \\
32 & 256 & 17.3717 & 14.2691 \\
128 & 256 & 15.7498 & 14.7510 \\
512 & 256 & 11.7709 & 10.3013 \\
2048 & 256 & \textbf{10.3837} & 10.0160 \\
8192 & 256 & 10.5371 & \textbf{9.5313} \\
16384 & 256 & 11.9027 & 10.7297 \\
\midrule
0 & 0 & \textbf{10.4122} & \textbf{8.3081} \\
32 & 0 & 10.4014 & 8.3460 \\
128 & 0 & 10.9531 & 8.3633 \\
512 & 0 & 11.1941 & 8.6657 \\
2048 & 0 & 12.1078 & 9.3075 \\
8192 & 0 & 23.7752 & 12.3623 \\
16384 & 0 & 27.1459 & 21.9300 \\
\bottomrule
\end{tabular}
\begin{minipage}{\linewidth}
\footnotesize
\textsuperscript{*}PPL for sequences with length $\geq$ n\_query
\end{minipage}
\label{tab:ppl_K-means}
\end{table}

\begin{table}[h!]
\centering
\caption{PPL comparison for \textbf{K-median} across configurations.}
\begin{tabular}{lrrrr}
\toprule
\textbf{Top K} & \textbf{Sample Size} & \textbf{PPL} & \textbf{PPL\textsuperscript{*}} \\
\midrule
0 & 256 & 17.5435 & 13.4139 \\
32 & 256 & 17.5589 & 14.3638 \\
128 & 256 & 15.1726 & 14.2688 \\
512 & 256 & 12.6928 & 10.7822 \\
2048 & 256 & \textbf{10.4396} & 10.6784 \\
8192 & 256 & 10.5228 & \textbf{9.6705} \\
16384 & 256 & 12.0311 & 10.6668 \\
\midrule
0 & 0 & \textbf{10.4122} & \textbf{8.3081} \\
32 & 0 & 10.5020 & 8.3319 \\
128 & 0 & 10.6929 & 8.3912 \\
512 & 0 & 10.9729 & 8.5140 \\
2048 & 0 & 11.6279 & 8.9240 \\
8192 & 0 & 20.9296 & 12.3335 \\
16384 & 0 & 22.5637 & 18.5101 \\
\bottomrule
\end{tabular}
\begin{minipage}{\linewidth}
\footnotesize
\textsuperscript{*}PPL for sequences with length $\geq$ n\_query
\end{minipage}
\label{tab:ppl_K-median}
\end{table}

\begin{table}[h!]
\centering
\caption{PPL comparison for \textbf{Leverage Score-Based Method} across configurations.}
\begin{tabular}{lrrrr}
\toprule
\textbf{Top K} & \textbf{Sample Size} & \textbf{PPL} & \textbf{PPL\textsuperscript{*}} \\
\midrule
0 & 256 & 17.5428 & 13.4129 \\
32 & 256 & 21.3402 & 15.7013 \\
128 & 256 & 21.3568 & 17.0048 \\
512 & 256 & 15.0292 & 13.4522 \\
2048 & 256 & 11.4189 & 9.8549 \\
8192 & 256 & \textbf{10.6066} & \textbf{9.4091} \\
16384 & 256 & 12.1050 & 10.5868 \\
\midrule
0 & 0 & \textbf{10.4122} & \textbf{8.3081} \\
32 & 0 & 10.6251 & 8.4462 \\
128 & 0 & 11.1715 & 8.4504 \\
512 & 0 & 11.4453 & 8.6360 \\
2048 & 0 & 12.3255 & 9.1102 \\
8192 & 0 & 23.7757 & 13.4991 \\
16384 & 0 & 30.7175 & 28.9817 \\
\bottomrule
\end{tabular}
\begin{minipage}{\linewidth}
\footnotesize
\textsuperscript{*}PPL for sequences with length $\geq$ n\_query
\end{minipage}
\label{tab:ppl_lev}
\end{table}

\section{Appendix B: Counterexample for K-Means Sensitivities}\label{appb}

This counterexample demonstrates that k-means clustering can fail to identify the set $S$ of important keys, even when the data satisfies the orthogonality conditions ($\delta_1, \delta_2 \to 0$) of the planted model from LevAttention. We show this failure is due to the sensitivity of k-means to large deviations in data point norms, a problem that our row-norm regularization assumption explicitly prevents.

\subsection{Setup of the Counterexample}
Let $d$ be the feature dimension. We construct an $n \times d$ matrix $K$ with $n \gg d$. We define the set of ``relevant keys" $S \subset [n]$ with $|S| = d/2$ (assuming $d$ is an even integer). The remaining $n-|S|$ keys form the set $S^c = [n] \setminus S$. We define the rows of $K$ as follows:
\begin{enumerate}
    \item \textbf{For $j \in S$:} Let $S = \{1, 2, \ldots, d/2\}$. For each $j \in S$, $K_j$ is a standard basis vector with unit Euclidean norm, supported on the first $d/2$ coordinates. $$ K_j = \mathbf{e}_j = (\underbrace{0, \ldots, 0}_{j-1}, 1, 0, \ldots, 0, \underbrace{0, \ldots, 0}_{d/2}) \in \mathbb{R}^d $$ where the single $1$ is in the $j$-th position. So, $||K_j||_2^2 = 1$.
    \item \textbf{For $l \in S^c$:} Let $S^c = \{d/2+1, \ldots, n\}$. For each $l \in S^c$, $K_l$ is a vector with a large Euclidean norm $M \gg 1$, supported on the remaining $d/2$ coordinates. For simplicity, we  assume all of the $K_l$ for $l \in S^c$ are identical and non-zero only in the $(d/2+1)$-th coordinate. $$ K_l = (0, \ldots, 0, M, 0, \ldots, 0) \in \mathbb{R}^d $$ where the value $M$ is in the $(d/2+1)$-th position. So, $||K_l||_2^2 = M^2$.
\end{enumerate}

\subsection{Verification of Planted Model Assumptions}
We check if the counterexample satisfies the planted model assumptions (1) and (2) of LevAttention for small $\delta_1, \delta_2$.
\begin{itemize}
    \item \textbf{Assumption (1): $\forall j,l \in S, j \ne l, |K_{j}K_{l}^{T}|\le\delta_{1}\cdot \min(||K_{j}||_{2}^{2},||K_{l}||_{2}^{2})$} For $j, l \in S$ and $j \ne l$, $K_j = \mathbf{e}_j$ and $K_l = \mathbf{e}_l$. Since $j \ne l$, $K_j K_l^T = \mathbf{e}_j^T \mathbf{e}_l = 0$. The minimum norm is $\min(||K_j||_2^2, ||K_l||_2^2) = \min(1, 1) = 1$. Thus, $0 \le \delta_1 \cdot 1$. We can choose $\delta_1 = 0$, which satisfies the assumption. 
    
    \item \textbf{Assumption (2): $\forall j\in S,l\notin S,|K_{l}K_{j}^{T}|\le\delta_{2}\cdot \min(||K_{j}||_{2}^{2},||K_{l}||_{2}^{2})$} For $j \in S$, $K_j$ is supported on the first $d/2$ coordinates. For $l \notin S$, $K_l$ is supported on the $(d/2+1)$-th coordinate. Therefore, $K_j K_l^T = 0$. The minimum norm is $\min(||K_j||_2^2, ||K_l||_2^2) = \min(1, M^2) = 1$ (since $M \gg 1$). Thus, $0 \le \delta_2 \cdot 1$. We can choose $\delta_2 = 0$, again satisfying the asumption. 
\end{itemize}
This simple example demonstrates that $\delta_1$ and $\delta_2$ can be zero, implying perfect orthogonality between points in $S$, and between points in $S$ and points in $S^c$.

\section{Appendix C: Supplementary Proofs}
\label{appc}

\subsection{Proof of Lemma~\ref{lem1}}
\begin{proof}
By the Cauchy–Schwarz inequality, for any unit vector $x$, $(A_i^\top x)^2 \le \|A_i\|^2$, so
$h_i = \sup_{\|x\|=1}\frac{(A_i^\top x)^2}{\|A x\|^2} \le \frac{\|A_i\|^2}{\sigma_{\min}^2}.$
Under the Gaussian noise model, $\|A_i\|^2 \approx d\,\sigma_N^2 = d\,(c_N/n\epsilon)$. As established in the proof of Theorem~\ref{thm1}, standard matrix concentration (given in LevAttention) ensures that $\sigma_{\min}^2=\lambda_{\min}(A^\top A)=\Theta(1/\epsilon)$ and does not decay with $n$. Since $n\gg d/\epsilon$, we obtain $h_i = O(d/n)$.
\end{proof}


\subsection{Proof of Lemma~\ref{lem2}}
\begin{proof}
Since $A_i = v_j + \delta_{i,j}$ and $\|v_j\|=1$, we have $(A_i^\top v_j)^2 = \bigl(1 + \delta_{i,j}^\top v_j\bigr)^2 \approx 1$ (up to $O(\|\delta_{i,j}\|)$). Meanwhile, $\|A v_j\|^2 = \sum_{\ell=1}^n (A_\ell^\top v_j)^2 \approx \sum_{i\in S_j}1 = m = \lceil\frac1\epsilon\rceil$. Hence $h_i \ge \frac{1}{m} = \Theta(\epsilon)$.
\end{proof}

Recall that for any row $i$, its leverage score is
$h_i = A_i\,(A^\top A)^{-1}A_i^\top = \sup_{\|x\|=1}\frac{(A_i^\top x)^2}{\|A x\|^2}.$

\subsection{Proof of Theorem~\ref{thm1}}
\begin{proof}[Proof sketch]
Since $\|A_i\|_2=1$ and $\sigma_{\min}^2=\Theta(1/\varepsilon)$,
\[
h_i = A_i^\top (A^\top A)^{-1}A_i \le \frac{\|A_i\|_2^2}{\sigma_{\min}^2}
= \frac{1}{\Theta(1/\varepsilon)} = O(\varepsilon),
\]
so $\max_{i\in N} h_i \le C_{\mathrm{noise}}\varepsilon$.
By Lemma~\ref{lem2}, each signal row satisfies
$h_i \ge C_{\mathrm{sig}}\varepsilon$, completing the proof.
\end{proof}

\subsection{Explanation of Theorem~\ref{thm2}}

The key to $k$-means recovering the planted structure is cluster separability: each “signal" row with a large leverage score must lie far enough from every other row that allocating it its own centroid strictly lowers the within-cluster distortion. In our analysis we therefore initialize $k=\Theta(d/\epsilon)$ centroids—one for each row whose leverage score is $\Omega(\epsilon)$, plus a single centroid that absorbs all residual (low-score) rows. Under this choice the clusters are provably well-separated, and the standard $k$-means objective attains its global minimum exactly at the desired partition, yielding a clean solution for the instance at hand.

\medskip

\subsection{Proof of Claim \ref{claim1}} \label{claim1_proof}
\begin{proof}
We follow three steps: (1) compute inter-centroid $\ell_p$ separations, (2) bound intra-cluster $\ell_p$ variances, (3) invoke well-separated cluster recovery for $\ell_p$-$k$-means.

\medskip
\noindent\textbf{1. True centroid positions and inter-cluster distances.}\\
The signal centroids are 
$$
v_1,\dots,v_d,\quad\text{and}\quad 0\in\mathbb R^d,
$$
with $\|v_j\|_p^p = 1$ (one coordinate of magnitude 1, rest zero), and 
$$
\|v_j - v_k\|_p^p 
= |1-0|^p + |0-1|^p = 2,\quad j\neq k,
\quad
\|v_j - 0\|_p^p = 1.
$$
Hence the minimum inter-centroid distance (in $p$-th power) is 
$$
\Delta_{\min}
= \min\bigl\{\|v_j - 0\|_p^p,\|v_j - v_k\|_p^p\bigr\}
= 1.
$$

\medskip
\noindent\textbf{2. Intra-cluster $\ell_p$ variances.}\\
Fix any signal cluster $S_j$.  For $i\in S_j$, 
$$
A_i = v_j + \delta_{i,j}, 
\quad \delta_{i,j}\sim\mathcal N(0,\sigma_S^2 I_d).
$$
We need 
$$
\mathbb E\bigl[\|A_i - v_j\|_p^p\bigr]
= \mathbb E\bigl[\|\delta_{i,j}\|_p^p\bigr].
$$
By standard moment bounds for a Gaussian vector in $\mathbb R^d$ (e.g.\ Rosenthal-type inequalities), there is a constant $C_p$ so that
$$
\mathbb E\bigl[\|\delta_{i,j}\|_p^p\bigr]
= \Theta\bigl(d\,\sigma_S^p\bigr)
= \Theta\bigl(d\,(c_S/d)^{p/2}\bigr)
= O(c_S^{p/2}\,d^{1-p/2}).
$$
Since $m=|S_j|=\lceil1/\epsilon\rceil$, by concentration of i.i.d.\ sums (via Rosenthal + Markov), with probability $1-e^{-\Omega(m)}$,
$$
\max_{i\in S_j}\|A_i - v_j\|_p^p
\;\le\; O\bigl(c_S^{p/2}\,d^{1-p/2}\bigr)
\;\ll\; 1
$$
whenever $c_S$ is small

Similarly for the noise cluster $S_0$, each $A_i=\eta_i\sim\mathcal N(0,\sigma_N^2I)$ gives
\[
\begin{split}
\mathbb E\bigl[\|\eta_i\|_p^p\bigr]
&= \Theta\bigl(d\,\sigma_N^p\bigr) \\
&= \Theta\bigl(d\,(c_N/(n\epsilon))^{p/2}\bigr) = O\bigl((c_N/(n\epsilon))^{p/2}\,d\bigr),
\end{split}
\]
and by a Markov bound with constant probability \footnote{See \textit{Computing Apporximate $l_p$ Sensitivities, \cite{padmanabhan2023}} for detail}
\[
\max_{i\in S_0}\|\eta_i\|_p^p\ll1
\]
if $c_N$ is small and $n\gg d/\epsilon$.

Thus the maximum within-cluster $p$-power deviation is  
\[
\begin{split}
\delta_{\max} &:= \max\Bigl\{\max_{i\in S_1\cup\cdots\cup S_d}\|A_i - v_{\mathrm{true}(i)}\|_p^p,\,
\max_{i\in S_0}\|\eta_i\|_p^p\Bigr\} \\
&= o(1)\,,
\end{split}
\] 
where $\mathrm{true}(i)$ is the signal-index for row $i$.

\medskip
\noindent\textbf{3. Well-separated clustering and exact recovery.}\\
Consider the $k$-means objective under $\ell_p$-powers:
$$
\min_{\substack{C_0,\dots,C_d\\\mu_0,\dots,\mu_d}}
\sum_{j=0}^d \sum_{i\in C_j}\|A_i - \mu_j\|_p^p.
$$
We compare the cost of the true partition $\{S_0,\dots,S_d\}$ with any incorrect partition that assigns some point $i_*\in S_j$ to the wrong cluster $C_k$, $k\neq j$.  Since its true centroid $v_j$ and the wrong centroid $v_k$ satisfy
$$
\|A_{i_*}-v_k\|_p^p
\;\ge\; \big(\,\|v_j - v_k\|_p - \|A_{i_*}-v_j\|_p\,\big)^p
\;\ge\; \big(2^{1/p}-\delta_{\max}^{1/p}\big)^p.
$$
whereas the cost at the correct centroid is
$\|A_{i_*} - v_j\|_p^p \le \delta_{\max}.$
Thus the extra cost of misplacing any single point is at least
$$
(2-\delta_{\max}) - \delta_{\max}
= 2 - 2\delta_{\max} > 0
$$
provided $\delta_{\max}<1$.  A symmetric argument holds for mis-assigning a noise point from $S_0$ to any $S_j$. Because moving any point to a wrong cluster strictly increases the total cost, the unique global minimizer of the $k$-means objective is the true partition $\{S_0,\dots,S_d\}$.  

This completes the proof of Claim \ref{claim1}.
\end{proof}

\section{Appendix D: Formal Argument: Identifying Set S with Optimal K-Means under Planted Model}\label{appd}

This argument demonstrates how an optimal $k$-means clustering can be expected to identify the set of relevant keys $S$ (by making each $K_j \in S$ a singleton cluster) when the data adheres to the specific geometric structure described in the planted model of Section 4 of LevAttention\cite{kannan2024}.

\subsection{Context and Assumptions}

The analysis relies on the following key aspects from the paper:
\begin{enumerate}
    \item \textbf{Key Matrix $K$:} An $n \times d$ matrix where each row $K_i$ is a key vector.
     \item \textbf{Planted Model Structure:} The model assumes the existence of a latent subspace $V$ (dimension $k_{latent} \le d/4$). Keys are distributed based on their relation to $V$ and its orthogonal complement $V^\perp$:
    \begin{itemize}
         \item For $i \in S$: $K_i$ is primarily within $V$, with a small "noise" component in $V^\perp$.
         \item For $i \notin S$: $K_i$ is primarily within $V^\perp$, with a small component in $V$.
    \end{itemize}
    \item \textbf{Row Norms:} We simplify by assuming $||K_i||_2^2 \approx c$ for some constant $c>0$ for all $i \in [n]$. This prevents the issue demonstrated in the counterexample where disparate norms bias $k$-means.
     \item \textbf{Leverage Score Relation:} The paper defines a threshold $\rho = \frac{1}{1+\delta_1^2|S|+\delta_2^2n}$ such that the set of points with leverage score $LS(K_i) \ge \rho$ forms a "universal set" $\tilde{U}$ of size $|\tilde{U}|  \le d/\rho$, and $S \subseteq \tilde{U}$. This implies that for smaller $\delta_1$ and $\delta_2$, $\rho$ tends to be larger, leading to a smaller $\tilde{U}$.
\end{enumerate}

\subsection{Implied Geometric Isolation by the Planted Model}

The structure from Section 4 of LevAttention, combined with small $\delta_1$ and $\delta_2$, dictates the Euclidean distances between key vectors. The squared Euclidean distance is $||K_a - K_b||_2^2 = ||K_a||_2^2 + ||K_b||_2^2 - 2K_a K_b^T$. Assuming $||K_i||_2^2 \approx c$:

\begin{enumerate}
    \item \textbf{Between distinct points within $S$ ($K_j, K_{j'} \in S, j \ne j'$):}
    From ($P_1$), $|K_j K_{j'}^T| \le \delta_1 c$. Thus, $K_j K_{j'}^T \in [-\delta_1 c, \delta_1 c]$.
    The squared distance is $||K_j - K_{j'}||_2^2 = 2c - 2K_j K_{j'}^T$.
    Therefore, $2c(1-\delta_1) \le ||K_j - K_{j'}||_2^2 \le 2c(1+\delta_1)$.
    If $\delta_1$ is sufficiently small (e.g., $\delta_1 \to 0$), $K_j K_{j'}^T \to 0$, implying $K_j$ and $K_{j'}$ are nearly orthogonal. The distance approaches $2c$.

    \item \textbf{Between a point in $S$ and a point not in $S$ ($K_j \in S, K_l \notin S$):}
    From ($P_2$), $|K_j K_l^T| \le \delta_2 c$. Thus, $K_j K_l^T \in [-\delta_2 c, \delta_2 c]$.
    The squared distance is $||K_j - K_l||_2^2 = 2c - 2K_j K_l^T$.
    Therefore, $2c(1-\delta_2) \le ||K_j - K_l||_2^2 \le 2c(1+\delta_2)$.
    If $\delta_2$ is sufficiently small (e.g., $\delta_2 \to 0$), $K_j K_l^T \to 0$, implying $K_j$ and $K_l$ are nearly orthogonal. The distance approaches $2c$.
\end{enumerate}
This analysis confirms that if $\delta_1$ and $\delta_2$ are sufficiently small, any point $K_j \in S$ is approximately orthogonal to all other points $K_x$ in the dataset (for $x \ne j$). This means points in $S$ are **geometrically isolated** from each other and from points in $S^c$. The minimum squared distance from any $K_j \in S$ to any other point $K_x$ is $\min(2c(1-\delta_1), 2c(1-\delta_2))$. Let this minimum be $D_{sep}$. If $\delta_1, \delta_2 < 1$, then $D_{sep} > 0$.

\subsection{Optimal K-Means Clustering}

The $k$-means algorithm aims to minimize the total sum of squared Euclidean distances from each data point to its cluster's centroid: $\sum_{C \in \mathcal{C}} \sum_{x \in C} ||x - \mu_C||_2^2$.

Consider a hypothetical partition $\mathcal{P}^*$ where:
\begin{enumerate}
    \item Each $K_j \in S$ is assigned to its own singleton cluster, $\{K_j\}$. The contribution of these $|S|$ clusters to the total distortion is $\sum_{K_j \in S} ||K_j - K_j||_2^2 = 0$.
    \item The remaining $n-|S|$ points ($K_l \notin S$) are optimally clustered into the remaining $k-|S|$ clusters.
\end{enumerate}

Now, consider an alternative partition $\mathcal{P}'$ where at least one point $K_j^* \in S$ is *not* a singleton, meaning it's grouped with at least one other point $K_x \ne K_j^*$. Let $C^*$ be the cluster containing $K_j^*$.
The contribution of $K_j^*$ to the distortion of $C^*$ is $||K_j^* - \mu_{C^*}||_2^2$. Since $C^*$ contains at least two points, $\mu_{C^*}$ will not be $K_j^*$, and thus $||K_j^* - \mu_{C^*}||_2^2 > 0$.
A lower bound for the distortion of a cluster containing $K_j^*$ and $K_x$ is $\frac{1}{2}||K_j^* - K_x||_2^2$. This is at least $\frac{1}{2} D_{sep}$.

For $\mathcal{P}'$ to be optimal, its total distortion must be less than or equal to that of $\mathcal{P}^*$. However, by forcing $K_j^*$ into a non-singleton cluster, $\mathcal{P}'$ incurs a minimum distortion of at least $\Delta_{min} = D_{sep}/2$ for $K_j^*$ (assuming $K_j^*$ is grouped with only one other point, and that point is optimal). In contrast, $\mathcal{P}^*$ achieves 0 distortion for $K_j^*$.

If $\Delta_{min}$ is sufficiently large (i.e., $\delta_1, \delta_2$ are very small), the penalty incurred by grouping an $S$ point with any other point is substantial. An optimal $k$-means algorithm, seeking to minimize the overall sum of squares, will strongly prefer to assign each $K_j \in S$ to its own singleton cluster, provided there are enough available centroids ($k \ge |S|$).

\subsection{Relationship to Leverage Scores and Choices of k}

The geometric isolation of $K_j \in S$ (due to small $\delta_1, \delta_2$) directly corresponds to them having high leverage scores. A high leverage score signifies that a point is not redundant and is influential in determining the subspace spanned by the data.  The paper states that $S \subseteq \{i : LS(K_i) \ge \rho\}$.

 The total sum of leverage scores is $d$. This implies that only a limited number of points can have high leverage scores.  Specifically, the number of points with $LS(K_i) \ge \rho$ is bounded by $d/\rho$.

Choosing $k \approx O(d/\rho)$ for $k$-means is thus a natural choice. This value of $k$ represents the approximate number of "truly distinct" or "influential" components in the data. Since the points in $S$ are by definition isolated and influential (high leverage scores), and the points in $S^c$ with low leverage scores are redundant (as discussed in your intuition), optimal $k$-means will:
\begin{enumerate}
    \item Utilize $|S|$ of the $k$ available clusters to form singleton clusters for each $K_j \in S$, achieving 0 distortion for these points.
    \item Distribute the remaining $k-|S|$ clusters among the $n-|S|$ points in $S^c$. The low leverage scores of many $S^c$ points imply they are either of small norm or highly correlated with other points, making them amenable to being clustered efficiently with other $S^c$ points without significant distortion or "interfering" with the isolation of $S$.
\end{enumerate}

\subsection{Conclusion}

Under the assumptions of the planted model (Section 4 of LevAttention), where small $\delta_1$ and $\delta_2$ parameters ensure that points in $S$ are geometrically distinct and isolated from all other points in the dataset, an optimal $k$-means clustering will strongly favor assigning each $K_j \in S$ to its own singleton cluster. This is because creating singleton clusters for these points results in zero distortion for them, which cannot be improved upon. Given a sufficient number of clusters, e.g., $k \approx O(d/\rho)$, which aligns with the number of high-leverage points in the dataset, $k$-means will successfully identify the set $S$ by placing its elements into distinct singleton clusters. This approach leverages the geometric properties of the data that also underlie the concept of leverage scores.
\section{Appendix E: Baseline Performance of LevAttention on ViT}\label{appe}


\begin{table}[H]
\centering
\small
\caption{ViT ImageNet-1k validation accuracy from LevAttention.}
\label{tab:vit-attn-accuracies}
\begin{tabular}{l r}
\toprule
\textbf{Model} & \textbf{Top-1 Acc. (\%)} \\
\midrule
S/16 (softmax) & 76.47 \\
S/16 (LevAttn, top-32\textsuperscript{*}) & 13.3 \\
S/16 ($\ell_2$ norm, top-32\textsuperscript{*}) & 3.3 \\
S/16 (LevAttn, top-32) & 68.30 \\
S/16 (LevAttn, top-64) & 72.48 \\
\midrule
L/16 (softmax) & 78.83 \\
L/16 (LevAttn, top-32\textsuperscript{*}) & 48.58 \\
L/16 ($\ell_2$ norm, top-32\textsuperscript{*}) & 8.9 \\
L/16 (LevAttn, top-32) & 75.12 \\
L/16 (LevAttn, top-64) & 77.27 \\
L/16 (LevAttn, top-128) & 77.17 \\
\bottomrule
\end{tabular}
\vspace{2pt}

\end{table}



\section{Appendix F: Effects of Residual Coupling Artifacts (GLM2 Ablation)}
\label{appf}

\paragraph{Context.}
Early experiments (denoted GLM2) used an initial integration of pre-scoring with HyperAttention. Subsequent analysis revealed that the interaction between pre-scoring masks and the residual sampling path introduced unintended coupling effects. These effects do not reflect limitations of the pre-scoring principle itself, but rather implementation-level interactions between filtering and the approximate attention pipeline. The corrected integration (GLM3), used throughout the main text, resolves these issues.

\paragraph{Observed Non-Monotonic Behavior.}
GLM2 experiments exhibited a U-shaped perplexity curve as the retained key budget $k$ increased. This behavior arose from three coupling artifacts:

\begin{enumerate}
    \item \textbf{Zeroing of masked keys.}  
    Keys and values outside the pre-scored subset were physically zeroed before attention. This altered the geometry of the key space and caused zero vectors to collapse into shared LSH buckets, introducing artificial block interactions and increased variance at small $k$.

    \item \textbf{Residual reweighting mismatch.}  
    Residual Monte Carlo samples were scaled according to the global key count $n$ rather than the effective retained count $|S|$. When $k$ was small, this overweighted the residual contribution and amplified noise relative to the structured blockwise component.

    \item \textbf{Block–residual overlap.}  
    Keys could appear both in LSH blocks and in the residual sampling path, leading to double-counting of attention contributions. This artificially lowered perplexity at intermediate $k$ and contributed to the non-monotonic curve shape.
\end{enumerate}

These interactions distorted the efficiency–accuracy relationship and explain the GLM2 curve shape without indicating a weakness of pre-scoring itself.

\paragraph{Corrected Coupling (GLM3).}
The integration used in all main-text results removes these effects:

\begin{itemize}
    \item \textbf{Masking via attention bias.}  
    Instead of zeroing keys/values, a boolean mask is injected as a $-\infty$ bias inside the attention kernel, preserving the true geometry of the key distribution.

    \item \textbf{Residual scaling by effective key count.}  
    Residual samples are weighted by $|S|/\text{sample\_size}$, matching the number of valid keys after pre-scoring.

    \item \textbf{Explicit block–residual exclusion.}  
    Keys already used in blockwise attention are excluded from the residual path, preventing double-counting.
\end{itemize}

\paragraph{Effect of Corrections.}
With these adjustments, GLM3 results (reported in the main paper) exhibit a stable and nearly monotonic relationship between perplexity and key budget. At small $k$, pre-scoring primarily removes noisy or low-impact keys, acting as a denoising mechanism. As $k$ increases, additional keys introduce diminishing returns, consistent with the theoretical picture that most attention mass is concentrated in a limited subset of informative keys.

\begin{figure}[h!]
    \centering
    \graphicspath{{./pictures/}}  
    \includegraphics[width=\linewidth]{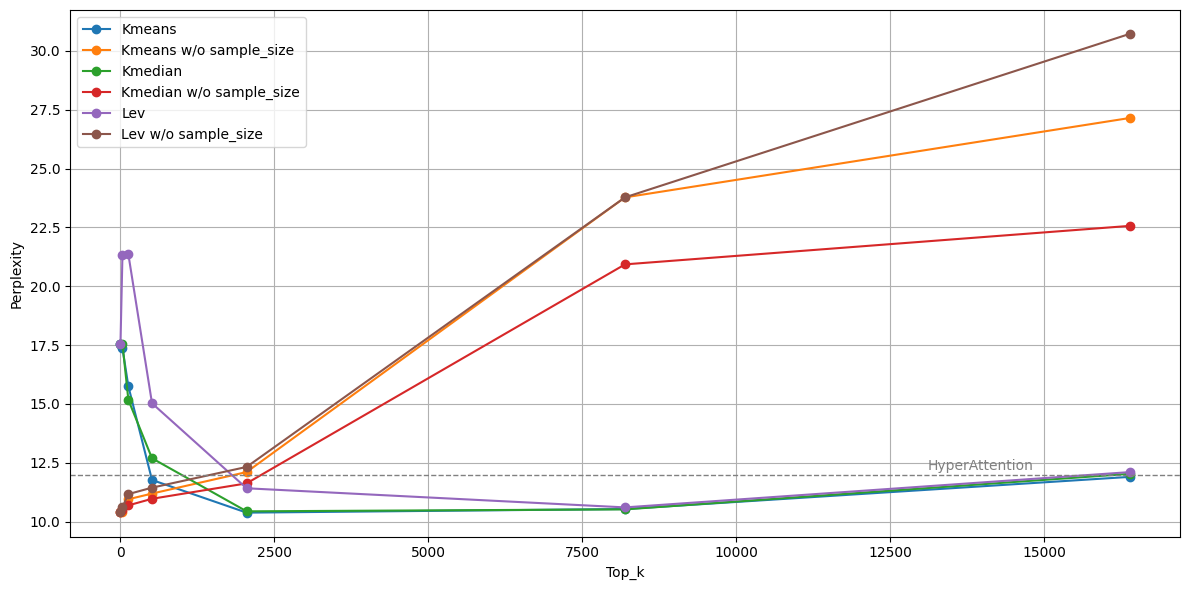}
    \caption{GLM2: The perplexity performance of various key selection strategies—K-means, K-median, and Leverage Score (Lev)—under different values of $k$ (the number of selected keys), where each $k$ is sampled as a power of $2$ (e.g., $128$, $2048$, $8192$).}
    \label{fig:image3}
\end{figure}

\paragraph{Interpretation.}
We therefore treat GLM2 results strictly as an \emph{ablation on coupling behavior} rather than as representative performance. All main-text curves correspond to GLM3, which reflects the intended algorithmic design. This separation clarifies that pre-scoring modifies which key–query interactions are evaluated, while the corrected integration ensures that no unintended interactions are introduced between filtering and residual approximation paths. Additional GLM2-only variants (e.g., Gaussian kernel k-means) are reported in Appendix~\ref{appi}.

\section{Appendix G: Heavy Attention Coverage Percentage}\label{appg}
We evaluated our custom attention mechanism that employs K-means and K-Median sampling to approximate the full attention distribution. This experiment was also conducted using ViT base model with patch 16 \cite{dosovitskiy2021vit} on Imagenet-1k validation set \cite{deng2009imagenet}.We compared the original attention output matrix with the attention output matrix after applying our clustering-based prescoring mechanism for analysis. We vary the number of sampled keys and adjust the threshold parameter $\epsilon$ (with values 0.01, 0.1, and 0.3) to measure the median percentage of heavy attention entries captured. An entry of attention matrix $A$ is considered heavy if $A_{ij}>\epsilon$ for $0\leq i,j \leq N$. From Figure~\ref{fig:median_percentage_kmeans} and Figure~\ref{fig:median_percentage_kmedian}, the capture percentage increases as $\epsilon$ or the number of keys sampled increases and K-Means has some marginal performance increase compared to K-Median. Additionally, we tested the number of top columns that contain the most heavy attention entries, and how well these can be captured by both sampling approaches. The result is shown in table\ref{tab:vit_results2}. While the overall heavy attention entries coverage rate shows a linear relationship, the top-$k$ coverage remains the same as the configurations change (where $k$ aligns with our number of keys sampled).
\begin{figure}[H]
  \centering
  \begin{minipage}{0.48\textwidth}
    \centering
    \begin{tikzpicture}
      \begin{axis}[
        width=5.2cm,
        height=5.2cm,
        xlabel={\normalsize Keys Sampled - K-Means},
        ylabel={\normalsize Percentage},
        grid=major,
        tick label style={font=\normalsize},
        label style={font=\normalsize},
        legend style={
  at={(1.05,0.5)}, 
  anchor=west,
  font=\normalsize
}
      ]
        \addplot[blue, thick, mark=*] coordinates {
           (32,22.57) (64,40.89) (128,70.54)
        };
        \addlegendentry{\(\epsilon=0.01\)}

        \addplot[green, thick, mark=triangle*] coordinates {
           (32,43.89) (64,69.07) (128,87.79)
        };
        \addlegendentry{\(\epsilon=0.1\)}

        \addplot[orange, thick, mark=diamond*] coordinates {
           (32,43.83) (64,69.89) (128,88.15)
        };
        \addlegendentry{\(\epsilon=0.3\)}
      \end{axis}
    \end{tikzpicture}
    \caption{K-Means: Median percentage vs. sampled keys.}
    \label{fig:median_percentage_kmeans}
  \end{minipage}\hfill
  \begin{minipage}{0.48\textwidth}
    \centering
    \begin{tikzpicture}
      \begin{axis}[
        width=5.2cm,
        height=5.2cm,
        xlabel={\normalsize Keys Sampled - K-Median},
        ylabel={\normalsize Percentage},
        grid=major,
        tick label style={font=\normalsize},
        label style={font=\normalsize},
        legend style={
  at={(1.05,0.5)}, 
  anchor=west,
  font=\normalsize
}
      ]
        \addplot[blue, thick, mark=*] coordinates {
           (32,22.56) (64,40.22) (128,69.54)
        };
        \addlegendentry{\(\epsilon=0.01\)}

        \addplot[green, thick, mark=triangle*] coordinates {
           (32,43.91) (64,64.69) (128,81.87)
        };
        \addlegendentry{\(\epsilon=0.1\)}

        \addplot[orange, thick, mark=diamond*] coordinates {
           (32,43.81) (64,69.74) (128,87.64)
        };
        \addlegendentry{\(\epsilon=0.3\)}
      \end{axis}
    \end{tikzpicture}
    \caption{K-Median: Median percentage vs. sampled keys.}
    \label{fig:median_percentage_kmedian}
  \end{minipage}
\end{figure}

\begin{table}[H]
  \caption{Top-k Heavy Columns Coverage}
  \label{tab:vit_results2}
  \centering
  \begin{tabular}{ll}
    \toprule
    Number of Keys Sampled & Average Percentage\\
    \midrule
    \textbf{Kmeans-32} & 15.62\% \\

    \midrule
    \textbf{Kmeans-64} & 32.81\% \\
    \midrule
    \textbf{Kmeans-128} & 65.62\% \\
    \midrule
    \textbf{Kmedian-32} & 18.75\% \\

    \midrule
    \textbf{Kmedian-64} & 32.81\% \\
    \midrule
    \textbf{Kmedian-128} & 65.62\% \\
    \bottomrule
  \end{tabular}
\end{table}

\section{Appendix H: Limitations and Future Work}\label{apph}

Our pre-scoring mechanism improves the accuracy--efficiency trade-off of approximate attention, but it comes with several limitations.

\paragraph{Overhead and hardware efficiency.}
Clustering-based pre-scoring introduces additional overhead (e.g., $O(nd\cdot k \cdot I)$ for $I$ clustering iterations with $k\ll n$), which motivates future work on more hardware-friendly implementations (e.g., minibatch/streaming clustering, quantized distance computations, and fused kernels). The resulting sparse access patterns may also reduce memory coalescing relative to dense attention, suggesting the need for GPU/TPU-aware masking and indexing strategies.

\paragraph{Model assumptions vs.\ real key distributions.}
Our guarantees rely on a planted-subspace abstraction and regularity conditions (e.g., small correlations and controlled row norms). While empirical results suggest robustness beyond the idealized setting, adversarial or highly skewed key distributions may violate these assumptions and degrade clustering quality. Designing adaptive normalization or robust pre-scoring rules that handle extreme outliers remains an important direction.

\paragraph{Baselines and evaluation scope.}
Our experiments focus on full-layer replacement within the HyperAttention pipeline and direct comparison to LevAttention, which are the most structurally aligned baselines for query-independent key prioritization under a fixed interaction budget. Other efficient-attention families---including kernelized approximations (e.g., Performer) and hashing-based methods (e.g., Reformer)---as well as streaming/KV-cache pruning approaches (e.g., attention-sink or heavy-hitter retention policies) target different regimes and constraints, and are not directly matched to our evaluation protocol. A broader benchmarking study that spans these families and decoding-time cache policies is valuable future work.

\section{Appendix I: Gaussian Kernel K-means}\label{appi}

\begin{table}[H]
\centering
\caption{GLM2: PPL comparison for \textbf{Gaussian Kernel K-means} across configurations.}
\small
\setlength{\tabcolsep}{6pt}
\begin{tabular}{@{}lr@{\hspace{14pt}}lr@{}}
\toprule
\multicolumn{2}{c}{\textbf{Sample Size = 256}} & \multicolumn{2}{c}{\textbf{Sample Size = 0}} \\
\cmidrule(lr){1-2}\cmidrule(l){3-4}
\textbf{Top K} & \textbf{PPL} & \textbf{Top K} & \textbf{PPL} \\
\midrule
0     & 17.5410 & 0     & \textbf{10.4122} \\
32    & 18.5715 & 32    & 10.7094 \\
128   & 19.6897 & 128   & 11.6935 \\
512   & 15.6162 & 512   & 11.2801 \\
1024  & 14.3289 & 1024  & 11.8827 \\
2048  & 11.8682 & 2048  & 13.6853 \\
4096  & 10.5833 & 4096  & 17.6743 \\
8192  & \textbf{10.0670} & 8192  & 22.8829 \\
16384 & 11.8495 & 16384 & 34.6969 \\
\bottomrule
\end{tabular}
\label{tab:ppl_otherkmeans}
\end{table}
\FloatBarrier
We further tested Gaussian kernel k-means on GLM2, which computes distance to centroids via the kernel method. Results are reported in Table~\ref{tab:ppl_otherkmeans}. The best performance appears at \verb|top-k=8192| with \verb|Sample Size=256| (\verb|perplexity=10.06|). Similar to previous methods, perplexity reaches the lowest value at \verb|top-k=0| when \verb|Sample Size=0|.

\end{document}